\documentclass{article}

\usepackage{microtype}
\usepackage{graphicx}
\usepackage{subfigure}
\usepackage{booktabs} 
\usepackage{xspace}
\usepackage{stfloats}  

\usepackage{hyperref}

\usepackage[accepted]{icml2025}

\usepackage{amsmath}
\usepackage{amssymb}
\usepackage{mathtools}
\usepackage{amsthm}

\usepackage[capitalize,noabbrev]{cleveref}

\theoremstyle{plain}
\newtheorem{theorem}{Theorem}[section]

\theoremstyle{definition}

\theoremstyle{remark}

\usepackage[textsize=tiny]{todonotes}

\newcommand{\piref}{\pi_0}
\newcommand{\piopt}{\pi^*}
\newcommand{\Vs}{\mathcal{V}}

\newcommand{\Qs}{\mathcal{Q}}

\newcommand{\As}{\mathcal{A}}

\newcommand{\Rs}{\mathcal{R}}

\newcommand{\KL}{\operatorname{KL}}

\icmltitlerunning{Soft Policy Optimization}

\begin{document}

\twocolumn[
\icmltitle{Soft Policy Optimization: \\Online Off-Policy RL for Sequence Models} 



\icmlsetsymbol{equal}{*}

\begin{icmlauthorlist}
\icmlauthor{Taco Cohen}{equal,meta}
\icmlauthor{David W. Zhang}{equal,meta}
\icmlauthor{Kunhao Zheng}{meta}
\icmlauthor{Yunhao Tang}{meta}
\icmlauthor{Remi Munos}{meta}
\icmlauthor{Gabriel Synnaeve}{meta}
\end{icmlauthorlist}

\icmlaffiliation{meta}{Meta Platforms, Inc.}

\icmlcorrespondingauthor{Taco Cohen}{tscohen@meta.com}
\icmlcorrespondingauthor{David W. Zhang}{dwz@meta.com}

\icmlkeywords{Machine Learning, ICML}

\vskip 0.3in
]



\printAffiliationsAndNotice{\icmlEqualContribution} 

\begin{abstract}
RL-based post-training of language models is almost exclusively done using on-policy methods such as PPO. 
These methods cannot learn from arbitrary sequences such as those produced earlier in training, in earlier runs, by human experts or other policies, or by decoding and exploration methods.
This results in severe sample inefficiency and exploration difficulties, as well as a potential loss of diversity in the policy responses.
Moreover, asynchronous PPO implementations require frequent and costly model transfers, and typically use value models which require a large amount of memory.
In this paper we introduce Soft Policy Optimization (SPO), a simple, scalable and principled Soft RL method for sequence model policies that can learn from arbitrary online and offline trajectories and does not require a separate value model.
In experiments on code contests, we shows that SPO outperforms PPO on pass@10, is significantly faster and more memory efficient, is able to benefit from off-policy data, enjoys improved stability, and learns more diverse (i.e. soft) policies.

\end{abstract}

\section{Introduction}
\label{sec:introduction}
One of the key traits that distinguishes humans from other animals is our capacity for social learning and open-ended cultural evolution.
Through the invention of writing and the internet, AI is now able to leverage the product of our cultural evolution by pretraining, yielding near human-level understanding of many subjects.
Yet, the reinforcement learning stage, which holds the unique potential to surpass human-level performance, still requires repeating the exploration process with each new training run -- much like a species whose members are unable to learn from the discoveries and mistakes of others.

Indeed, although online on-policy RL methods such as PPO consistently outperform pure offline methods like DPO, they must re-discover good behaviour in each training run \citep{Xu2024-zy, Rafailov2023-tr, Schulman2017-vl}. 
For hard prompts (e.g. a coding challenge), it can easily take thousands of rollouts to find a correct response \citep{wu2024inferencescalinglawsempirical, Brown2024-gu}.
The resulting trajectory is then used to do a small clipped gradient update (lest the policy collapse), after which it is discarded.
In order to get unbiased gradients and guarantee convergence, the rollouts must be performed using standard sampling because alternative decoding or exploration methods result in off-policy trajectories.
Clearly, then, it would be highly desirable to have RL methods that can learn from arbitrary previous experiences, as it would allow us to leverage exploration methods during training, and re-use hard-to-discover solutions from previous training runs.
In this way, AI models become participants in the knowledge evolution process.

In this paper we introduce Soft Policy Optimization (SPO), a principled and effective asynchronous online off-policy soft RL method.
SPO leverages both online and offline data, and excels at learning policies with high entropy and high return.
A key idea behind our method is what we call the \emph{Cumulative Q-Parameterization}, which parameterizes the token-level soft action-value function in terms of policy and reference model log-probabilities.

This unification of policy and value function saves a significant amount of memory.
More importantly, we prove that a cumulatively parameterized Q-function satisfies both soft Bellman-consistency and path-consistency \citep{Nachum2017-is} \emph{by construction}, except at tokens where a reward is observed (in our case, only the last one).
Hence, neither Q-learning nor path-consistency losses are needed for non-terminal tokens.
We therefore explore a family of simple off-policy losses that we derive from first principles and which can be thought of equivalently as terminal Q-value regression or full-sequence path-consistency learning.

We further prove that if either the policy or the value function is optimal, then so is the other one.
This allows us to flexibly combine policy-based and value-based losses (on both online and offline trajectories),
with the understanding that all of them share the same optimum.

We implement both PPO and a number of SPO variants in an asynchronous distributed RL framework \citep{gehring2024rlef, Noukhovitch2024-lf} and evaluate their performance on the CodeContests and TACO competitive programming benchmarks \citep{li2022competition, li2023taco}.
Our experiments demonstrate that SPO consistently outperforms PPO in terms of pass@10 scores, indicating that SPO learns more diverse policies in accordance with the soft RL theory on which our method is based.
We also demonstrate that SPO can effectively leverage offline data in practice to speed up learning and yield improved results.

Furthermore, we perform a number of ablations to study the effect of different SPO loss variants.
To demonstrate the utility of combining different loss functions, we show that results can be improved by regressing the Q-values towards Monte-Carlo targets for a subset of tokens obtained via Pivotal Token Search \citep{Abdin2024-xc}.
Finally, because off-policy learning does not require frequent model transfer between compute nodes, we observe speedups of about 85\% compared to PPO.

\section{Theoretical Results}


\subsection{Setup \& Basic Results}

We consider a simple setup, where the policy $\pi(a | x)$ sees a prompt $x$ (e.g. a programming problem), responds with a sequence of tokens  (actions) $a = a_1, \ldots, a_T$, and then receives a deterministic reward $r(x, a) \in \{0, -1\}$, e.g. based on whether a solution passes all tests ($r=0$) or not ($r=-1$).
There is no observation beyond the prompt, no intermediate reward, no environment stochasticity, and no discounting.
The ideas and methods in this paper can be generalized, but this setup simplifies the exposition and matches our current experimental setup.

Following much recent work \cite{Jaques2016-ym, Ziegler2019-or, Rafailov2023-tr}, we consider the KL-regularized reward maximization problem (``Soft RL''), 
\begin{equation}
    \label{eq:objective}
    \mathcal{L}(\pi) = \mathbb{E}_\pi[r] - \beta \KL[\pi, \piref],
\end{equation}
where $\piref$ is a reference policy, typically an instruction-tuned LLM that is also used to initialize $\pi$.
It is well known that the optimum of $\mathcal{L}$ is given by
\begin{equation}
    \label{eq:piopt}
    \piopt(a \, | \, x) = \exp(r(x, a) / \beta) \, \piref(a \, | \, x) / Z.
\end{equation}
This can be shown by maximizing $\mathcal{L}(\pi)$ subject to a normalization constraint, or by recognizing $\mathcal{L}(\pi)$ as a variational bound on $\KL[\pi, \piopt]$ (explained shortly). 

When $\beta$ is small, the reward factor $\exp(r/\beta)$ will be close to $0$ for failure ($r=-1$) and equal $1$ for success ($r=0$), thus acting as a (soft) constraint.
Moreover, the target $\piopt$ will be approximately proportional to $\piref$ on the set of successful $a$, thus (theoretically) preserving the full diversity of correct responses in $\piref$ instead of hard-maxing the reward, which should benefit exploration.
Clearly, learning \emph{all} good responses to a query is a harder problem than learning only one good response, but should lead to a deeper understanding of the domain and a reduction in reliance on spurious correlations and memorization.
Soft RL, which learns a high-entropy policy, is also more aligned with generative modelling objectives than classical hard RL that converges to a deterministic policy.

Indeed, we can interpret $\piopt$ as a success-conditioned \emph{posterior} \citep{Levine2018-sd}.
In this interpretation we view $\piref(a|x)$ as a \emph{prior} over actions, and define the \emph{likelihood} of an auxiliary binary variable $p(o=1 | x, a) = \exp(r(x,a) / \beta)$ (well defined for $r \leq 0$).
We can think of $o=1$ as ``success'' or ``optimality''.
By Bayes rule, the posterior is proportional to likelihood times prior, so $p(a | x, o=1) = \piopt(a|x)$.

The normalizing constant $Z = \mathbb{E}_{\piref}[\exp(r / \beta)]$ can be expressed in log-space as $\beta \log Z = \mathbb{S}^\beta_{\piref}[r]$ using the generalized softmax operator:
\begin{equation}
    \begin{aligned}
    \mathbb{S}^\beta_{\piref(a|x)}[r] = \beta \log \sum_{a} \piref(a|x) \exp(r(x, a) / \beta)
    \end{aligned}
\end{equation}
Taking limits, we see that $\mathbb{S}^\beta_{\pi}[r]$ interpolates between the classical value function $V^\pi = \mathbb{E}_\pi[r] = \mathbb{S}^\infty_\pi[r]$ and the optimal value function $V^* = \max_a[r] = \mathbb{S}^0_{\pi}[r]$.
Hence, we can think of $\mathbb{S}^\beta_{\piref}[r]$ as a generalized ``soft'' value function:
\begin{equation*}
    \Vs = \mathbb{S}_{\piref(a | x)}^\beta[r] = \beta \log Z = \beta \log p(o=1 | x).
\end{equation*}
By expanding $\beta \KL[\pi, \piopt]$ and rearanging terms, we find
\begin{equation}
    \label{eq:vi-bound}
    \begin{aligned}
        \mathcal{L}(\pi) = V^\pi - \beta \KL[\pi, \piref]
        =
        \Vs - \beta \KL[\pi, \piopt],
    \end{aligned}    
\end{equation}
This remarkable equation, which we have not found in the Soft RL literature, can be understood as the fundamental equation of variational inference expressed in RL notation.
Because $\Vs$ is the (soft) value of $\piref$, it is a constant with respect to optimization, so that maximizing $\mathcal{L}(\pi)$ is equivalent to minimizing $\KL[\pi, \piopt]$, making it evident that $\piopt$ is indeed the optimum of $\mathcal{L}(\pi)$ as claimed before.

\subsection{Token-level Policy \& Value Functions}
Since we wish to train a sequence model, we consider the token-level policy and value functions.
Given a partial response $a_{\leq t} = a_1, \ldots, a_t$, we define the soft $\Qs$-value:
\begin{equation*}
    \Qs_t = \mathbb{S}_{\piref(a_{>t} | a_{\leq t}, x)}^\beta[r] = \beta \log p(o=1 | a_{\leq t}, x).
\end{equation*}
Note that $\Qs_0 = \Vs$.
Furthermore, we define the advantage
\begin{equation}
    \As_t = \Qs_t - \Qs_{t-1}
    \label{eq:advantage-value}
\end{equation}
The advantage of $a_t$ can be interpreted as the change in log-likelihood of $o=1$  (when following $\pi_0$) that results from appending $a_t$ to the sequence $a_{< t}$.

By applying the sum and product rules of probability to \autoref{eq:piopt}, we find that the token-level optimal policy can be expressed as:
\begin{equation}
        \piopt(a_t | a_{< t}, x) = \exp(\As_t / \beta) \, \piref(a_t | a_{< t}, x)
        \label{eq:piopt_t}
\end{equation}
Thus, the \emph{posterior} $\piopt$ is an easily computable function of the $\Qs$ or $\As$-value of the \emph{prior} $\piref$.

\subsection{Path- \& Bellman Consistency}

Taking logarithms of \autoref{eq:piopt_t} we find that
\begin{equation}
    \As_t = \beta \log \piopt(a_t | a_{< t}, x) - \beta \log \piref(a_t | a_{< t}, x)
    \label{eq:advantage-policy}
\end{equation}
Thus $\As_t$ tells us how much more likely $a_t$ is under $\piopt$ compared to $\piref$, and (by \autoref{eq:advantage-value}), how much more likely $o=1$ is after adding $a_t$ to the sequence $a_{< t}$.

Using the definition $\As_t = \Qs_t - \Qs_{t-1}$ we find that for any interval $t = t_1, \ldots, t_k$:
\begin{equation}
    \label{eq:path-consistency-def}
    \sum_{t = t_1}^{t_k} \As_t = \Qs_{t_k} - \Qs_{t_1 - 1},
\end{equation}
That is, analogous to the fundamental theorem of calculus, integrating (summing) the change (in log-prob) at each time step yields the total change (see also \citet{Jenner2022-es}).
This kind of equation, which relates the value function $\Qs$ and the difference of policy log-probs $\As$, has been called \emph{path consistency} \cite{Nachum2017-is}.


Consider now \autoref{eq:path-consistency-def} for the whole sequence (i.e. $t_1 = 1, t_k = T$).
Rearanging terms and using $\Qs_T = r$ we obtain a forward and backward expression for $\Qs_t$:
\begin{equation}
    \Qs_t = \Qs_0 + \sum_{t'=1}^{t} \As_{t'} = r - \sum_{t'=t}^T \As_{t'}
    \label{eq:Q-cumsum-A}        
\end{equation}
The forward equations says that the sequence of $\Qs_t$ values can be obtained by a cumulative sum of the advantages.

By definition of $\As_t$ (eq. \ref{eq:advantage-value}) we have $\Qs_{t+1} = \Qs_t + \As_{t+1}$.
Using this fact one can show that $\Qs$ satisfies the following Bellman-like consistency equation:
\begin{equation}
    \begin{aligned}
        \mathbb{S}^\beta_{\piref}[\Qs_{t+1}] 
        =
        \mathbb{S}^\beta_{\piref}[\As_{t+1}] + \Qs_{t} 
        =
        \Qs_t        
    \end{aligned}
    \label{eq:Q-Bellman}
\end{equation}
where the subscript $\piref$ refers to the next-token distribution $\piref(a_{t+1} | a_{\leq t}, x)$.
Like the Bellman equations in classical (hard) RL, this equation relates the values at the next time step (weighted by $\piref$) to the value at the current time step. (note that there is no instantaneous reward because we only consider terminal rewards in our setup).
Since $\mathbb{S}^\beta$ interpolates between $\mathbb{E}$ and $\max$, the soft Bellman equation generalizes both the Bellman evaluation and optimality equations. 

Path-consistency is a property that holds for \emph{every} trajectory or contiguous sub-trajectory, regardless of how it was produced. 
It relates the $\Qs$ values and policy log-probs, evaluated at specific token values $a_t$.
By contrast, Bellman consistency relates the $\Qs_t$-value of $a_{\leq t}$ to the $\Qs_{t+1}$ value for \emph{all possible next} tokens $a_{t+1}$, using a policy (in our case $\piref$) to weigh the different possibilities.

\subsection{The Cumulative Q-Parameterization}
\label{sec:cumulative-q-param}
So far, we have discussed relations between the log-probs and value functions associated with the optimal and reference policy.
Now, let us consider an arbitrary sequence policy $\pi_\theta(a \, | \, x) = \prod_t \pi_\theta(a_t \, |\, a_{< t}, x)$ parameterized by $\theta$.

Based on the expression of $\As_t$ as a difference of policy log-probs (\autoref{eq:advantage-policy}), we propose to estimate the advantage as
\begin{equation}
    \As_t^\theta \equiv \beta \left( \log \pi_\theta(a_{t} \,|\, a_{< t},x) - \log \piref(a_{t} \,|\, a_{<t}, x) \right).
    \label{eq:A-parameterization}
\end{equation}
Notice that when initializing $\pi_\theta \leftarrow \piref$, the advantage estimates start off at $0$.
Now, based on the expression of $\As_t$ as a difference of $\Qs$-values (\autoref{eq:advantage-value}), and using the result of \autoref{eq:Q-cumsum-A}, we define the {\bf cumulative $\Qs$-parameterization} as:
\begin{equation}
    \Qs_t^\theta \equiv \widehat{\Qs}_0 + \sum_{t'=1}^t \As_{t'}^\theta = \Qs_{t-1}^\theta + \As_t^\theta,
    \label{eq:Q-parameterization}
\end{equation}
That is, we compute all $\Qs_t^\theta$ as a cumulative sum of log-prob differences (advantages), plus an initial value estimate.
The initial value of the prompt $\Vs^\theta = \Qs_0^\theta = \widehat{\Qs}_0$ is estimated once before training using Monte-Carlo sampling from $\piref$ (see \autoref{sec:MC-targets}). 
One could also estimate $\Qs_0$ online using a running average or using $\pi_\theta$ (e.g. as a log-prob or difference at the last prompt token or summed over the prompt), but learning accurate generalizing value functions is hard.


As argued by \citet{Tang2023-yh}, parameterizing the value $\Qs_t^\theta$ in terms of previous value plus advantage is statistically advantageous, since $\Qs_{t-1}$ can share statistical power from all visits to $a_{< t}$ whereas $\As^\theta_{t}$ only needs to learn whether $a_t$ was a relative improvement (thus sidestepping the hard task of estimating the magnitude of $\Qs_{t}$).

\begin{theorem}
    $\Qs^\theta = \Qs$ iff $\pi_\theta = \piopt$ and $\widehat{\Qs}_0 = \Qs_0$.
    \label{thm:optimality}
\end{theorem}
\begin{proof}
    The cumulative Q-parameterization defines an invertible map from $\widehat{\Qs}_0$ and policy log-probabilities to $\Qs_t^\theta$-values, so either one determines the other.
    Moreover, by \autoref{eq:advantage-policy} and \autoref{eq:Q-cumsum-A}, the true / optimal value $\Qs$ and $\piopt$ are related by the same mapping.
    Hence, if $\Qs^\theta = \Qs$ then the policy is optimal (and trivially $\widehat{\Qs}_0 = \Qs_0)$, and vice versa.
\end{proof}
It follows from \autoref{thm:optimality} that we can freely combine policy-based and value-based RL methods for improving $\pi_\theta$ and $\Qs^\theta$, as they will converge on the same optimal $\theta^*$.
A natural idea would be to use temporal difference (TD) or path-consistency (PC) losses.
However, the following theorems show that Bellman consistency and path consistency are already satisfied by construction, making a loss unnecessary.

\begin{theorem}
    \label{thm:bellman}
    For any $\theta$ and any sequence of tokens $a_1, \ldots, a_T$, the cumulative Q-function $\Qs^\theta_t$ satisfies Bellman consistency for all $0 \leq t \leq T$.
    That is, 
    \begin{equation*}
        \Qs_t^\theta = \mathbb{S}^\beta_{\piref(a_{t+1} | a_{\leq t}, x)}[\Qs_{t+1}^\theta]
    \end{equation*}
\end{theorem}
\begin{proof}
    By definition, we have $\Qs^\theta_{t+1} = \Qs^\theta_t + \As^\theta_{t+1}$. Furthermore, as with the expectation and max operators, we can take out additive constants from the softmax, so 
    \begin{equation*}
        \mathbb{S}_{\piref}^\beta[\Qs^\theta_t + \As^\theta_{t+1}] = \mathbb{S}_{\piref}^\beta[\As^\theta_{t+1}] + \Qs^\theta_t
    \end{equation*}
    It remains to show that $\mathbb{S}_{\piref}^\beta[\As^\theta_{t+1}] = 0$:
    \begin{equation*}
        \begin{aligned}
            \mathbb{S}_{\piref}^\beta[\As^\theta_{t+1}]
            &=
           \beta \log \sum_{a_{t+1}} \piref(a_{t+1} | a_{\leq t}) \frac{\pi_\theta(a_{t+1}|a_{\leq t})}{\piref(a_{t+1}|a_{\leq t})} \\
            &= \beta \log \sum_{a_{t+1}} \pi_\theta(a_{t+1}|a_{\leq t}) = 0.
        \end{aligned}
    \end{equation*}
    (we left out conditioning on the prompt $x$ for brevity)
\end{proof}
Note that the theorem only speaks about \emph{internal} consistency between the $\Qs$-values produced by the model.
At time $t = T$ we obtain an external reward and since there is no guarantee that $\Qs^\theta_T = \Qs_t = r$ we will in general get an error at that time.
However, the intermediate TD errors $\Qs^\theta_t - \mathbb{S}_{\piref}^\beta[\Qs^\theta_{t+1}]$ are always $0$.
So although standard Q-learning and even methods such as Q-Transformer \citep{Chebotar2023-tc} that are used with transformers need to learn Bellman-consistency even for intermediate steps, our Q-parameterization is internally consistent by construction.

We have a similar result for path consistency:
\begin{theorem}
    \label{thm:path-consistency0}
    For any $\theta$ and any sequence of tokens $a_1, \ldots, a_T$, the cumulative Q-function $\Qs^\theta_t$ satisfies path consistency on any interval $t_1, \ldots, t_k$ with $t_1 \leq t_k \leq T$.
    That is, 
    \begin{equation*}
        \sum_{t = t_1}^{t_k} \As^\theta_t = \Qs^\theta_{t_k} - \Qs^\theta_{t_1 - 1},
    \end{equation*}
\end{theorem}
\begin{proof}
    Substitute the definition of $\Qs_t^\theta$ (\autoref{eq:Q-parameterization}) and cancel shared terms $\widehat{\Qs}_0^\theta + \sum_{t'=1}^{t_1 - 1} \As_{t'}^\theta$, leaving only $\sum_{t=t_1}^{t_k} \As_t^\theta$.
\end{proof}

In path consistency learning \citep{Nachum2017-is}, one learns a separate policy and value function, and attempts to enforce path consistency on random intervals by learning.
When using the cumulative Q-parameterization, there is no point to this because path consistency is guaranteed.


\begin{theorem}
    If $\Qs_T^\theta = \Qs_T$ for all $a$, then $\Qs^\theta = \Qs$.
\end{theorem}
\begin{proof}
    By induction, assume $\Qs_{t+1}^\theta = \Qs_{t+1}$ for all $a_{t+1}$.
    Then by \autoref{thm:bellman}, 
    \begin{equation*}
        \Qs_t^\theta = \mathbb{S}_{\piref}^\beta[\Qs_{t+1}^\theta] = \mathbb{S}_{\piref}^\beta[\Qs_{t+1}] = \Qs_t.
    \end{equation*}
    Hence $\Qs^\theta = \Qs$.
\end{proof}
Together with \autoref{thm:optimality}, this demonstrates that if we can learn the \emph{terminal} $\Qs_T^\theta$ for all sequences, we have learned the true value function and optimal policy.

\section{Soft Policy Optimization}

SPO is a hybrid online off-policy method, which means we can learn from on-policy trajectories as well arbitrary offline data. 
A key feature of SPO is the use of the cumulative Q-Parameterization (\autoref{sec:cumulative-q-param}), which means we compute $\Qs_t^\theta$ as the value of the prompt $\widehat{\Qs}_0$ (estimated by sampling from $\piref$ before trainig) plus the cumulative sum of log-prob differences $\As_t^\theta = \beta(\pi_\theta(a_t | a_{<t},x) - \log \piref(a_t | a_{>t}, x))$.
Thus, quite intuitively, when the policy log prob for $a_t$ exceeds the reference model log prob (i.e. it is considered a good action by $\pi_\theta$), we have a positive advantage for $a_t$ and an increase in value $\Qs_t^\theta$.




\subsection{Learning Objectives}
\label{sec:objectives}

The results of the previous section show that a cumulative Q function (\autoref{eq:A-parameterization} and \autoref{eq:Q-parameterization}) is always self-consistent in both the Bellman- and path-consistency sense.
What remains to be learned is consistency with observed rewards.

The terminal $\Qs_T$ is equal to the terminal reward $r(x, a)$.
Hence, we wish to make $\Qs_T^\theta$ satisfy (for all $x, a$):
\begin{equation} \label{eq:path-consistency}
    \Qs_T^\theta \equiv \widehat{\Qs}_0 + \sum_{t=1}^T \As_t^\theta \approx r,
\end{equation}
where $\Qs_T^\theta$ is defined by the cumulative Q-parameterization, meaning $\widehat{\Qs}_0$ is a Monte-Carlo estimate of the soft value of the prompt $x$ (computed prior to training), and $\As^\theta_t = \beta (\log \pi_\theta - \log \piref)_t$ (\autoref{eq:A-parameterization}).

By rearranging terms, we can derive from \autoref{eq:path-consistency} various equations, each of which can be turned into an off-policy learning objective by choosing a loss function.
Below we first discuss the choice of loss function (\autoref{sec:losses}), and then consider the various off-policy learning objectives (\autoref{sec:off-policy-objectives}).
Finally, we also discuss the use of Monte-Carlo regression targets (\autoref{sec:MC-targets}) and policy gradients (\autoref{sec:PG}).
Here we will simply present all options; in \autoref{sec:experiments} we will systematically compare them.


\subsubsection{Loss functions}
\label{sec:losses}

Below we discuss different regression problems, each of which is defined by a prediction $\hat{y}$ that we want to bring close to a target $y$ (for instance, regressing $\Qs_T^\theta$ towards $r$).
To measure the discrepancy, we can employ a squared loss $L(\hat{y}, y) = (\hat{y} - y)^2$.
Where $\hat{y}$ and $y$ are log-probabilities, we can also use a binary cross-entropy loss: $L(\hat{y}, y) = - x \hat{y} - (1 - x) \log(1 - \exp(\hat{y}))$, with targets $x = \exp y$.

{\bf Clipping:} In some cases, the prediction $\hat{y}$ (and target $y$) may be larger than $0$ even though a log-probability can never be.
This results in infinities in the cross-entropy, so we use clipping on both $y$ and $\hat{y}$.
We clip the prediction $\hat{y}$ close to zero, and propagate gradient straight-through, ignoring clipping on the backward pass.
We further modify the positive term of the cross-entropy from $-x \hat{y}$ to $x \operatorname{relu}(-\hat{y})$ so as to remove any incentive to increase $\hat{y}$ above $0$.

{\bf Warping:} For any injective (difference-preserving) map $\sigma$, $\sigma(u) = \sigma(v) \Rightarrow u = v$.
Hence, for any equation implied by \autoref{eq:path-consistency}, we can apply $\sigma$ to both sides and aim to make the resulting equation hold.
In our experiments, we consider specifically the case where $\sigma$ is the sigmoid function.

\subsubsection{Off-Policy Objectives}
\label{sec:off-policy-objectives}

{\bf Terminal Q-Regression:} Most obviously, we can regress $\Qs_T^\theta$ towards the observed reward using some loss function $L(\Qs_T^\theta, r)$.
In Q-learning and other TD-methods, one typically regresses the Q-function towards TD targets using a squared loss.
Since $\Qs_t$ has an interpretation as a log-probability $\Qs_t = \log p(o=1 \,|\, a_{\leq t}, x)$, we can also use a binary cross-entropy loss.
Since $\Qs_t^\theta$ may exceed $0$, we use clipping as described in the previous section.

{\bf Non-Terminal Q-Regression:} 
Rearranging terms in \autoref{eq:path-consistency}, we can also attempt to make $\Qs_t^\theta = \widehat{\Qs}_0 + \sum_{t'=1}^t \As_t^\theta \approx r - \sum_{t'=t+1}^T \As_{t'}^\theta$ hold.
That is, we regress $\Qs_t^\theta$ towards targets $\Rs_t$ defined as $\Rs_t = \operatorname{detach}(r - \sum_{t'=t+1}^T \As_{t'}^\theta)$.
We refer to this as \emph{reverse-Q targets} because $\Rs_t$ is just a time-reversed estimate of $\Qs_t$.
Intuitively, the target for $\Qs_t^\theta$ is the reward, adjusted for the quality (advantage) of later actions.
If we obtained a low reward but this can be explained by later actions, perhaps $\Qs_t^\theta$ should still take a relatively high value.

When using squared loss $L = \sum_{t=1}^T (\Qs_t^\theta - \Rs_t)^2$, the error at each term is equal to the terminal one: $\Qs_t^\theta - \Rs_t = \Qs_T^\theta - r$.
However, because we detach the gradients of the targets $\Rs_t$, the gradients for the $t$-th term only flow into $\Qs_t^\theta$, leading to larger gradients for earlier terms.
Moreover, when using cross-entropy loss, the error itself is different for each term because the target $\Rs_t$ and predictor $\Qs_t$ enter into the cross-entropy formula in exponentiated and non-exponentiated form, respectively.

{\bf Advantage Regression:}
Finally, we can subtract $\widehat{\Qs}_0$ from \autoref{eq:path-consistency} to obtain $\sum_{t=1}^T \As_t^\theta = r - \widehat{Q}_0$.
Again the error is equivalent to terminal Q-regression, and in this case the gradients are also equivalent. 
Hence, we only consider advantage regression in combination with sigmoid-warping and a binary cross-entropy loss.

\subsubsection{Monte-Carlo Targets}  
\label{sec:MC-targets}

As mentioned, we use Monte-Carlo sampling to estimate $\widehat{Q}_0$.
Additionally, we experiment with directly regressing $\Qs_t^\theta$ towards a Monte-Carlo estimate of $\Qs_t$.
Recall that $\Qs_t = \mathbb{S}_{\piref}^\beta[r] = \beta \log \mathbb{E}_{\piref}[\exp{(r / \beta)}]$, i.e. it is the logarithm of the expectation of the exponentiated reward.
Approximating the expectation using samples from $\piref$ leads to a consistent but biased estimation of $\Qs_t$, due to the logarithm.
However, when using the cross-entropy loss, we get an unbiased estimate of the loss as long as the \emph{exponentiated} value targets are unbiased, which they are.

In order to estimate $\widehat{Q}_0$, we take $800$ samples from $\piref$ for each prompt, before training.
Since we assume our reward to be binary with $r \in \{0, -1\}$, we can estimate the success probability $\widehat{\mathcal{S}}_0 \approx p(r = 0 | x)$ and compute $\exp{(\Qs_0 / \beta)} = \mathbb{E}_{\piref}[\exp(r / \beta)] \approx \widehat{\mathcal{S}}_0 \exp(0/\beta)+ (1 - \widehat{\mathcal{S}}_0) \exp(-1 / \beta)$.


In order to test the utility of regressing $\Qs_t^\theta$ directly towards Monte-Carlo estimates, we annotate some of our offline data (described in \autoref{sec:experiments}) with estimates of the success probability $\mathcal{S}_t$ (from which a $\Qs_t$-estimate can be derived as above).
Since samples are taken from the reference model, early tokens in successful but improbable trajectories may produce failed rollouts with overwhelming probability, thus providing little learning signal.
Hence, we wish to focus the computational effort on ``interesting'' tokens, where the success probability changes significantly.

For this purpose, we use Pivotal Token Search as described in \citet{Abdin2024-xc}.
This algorithm searches for pivotal tokens (where the success probability changes significantly) by recursively bisecting the trajectory, performing $K$ rollouts to estimate $\mathcal{S}_t$ at the bisection point, and recursing if the endpoints of the current interval have sufficiently different $\mathcal{S}_t$.
In this way, a number of pivotal tokens are produced, and more importantly for our purposes, a sparse set of $\mathcal{S}_t$-estimates that can be used as targets for $\Qs_t^\theta$.

\subsubsection{Proximal Policy Optimization}
\label{sec:PG}



Our primary baseline is PPO \cite{Schulman2017-vl}, implemented following the approach in \citep{gehring2024rlef}.
PPO uses GAE \citep{Schulman2016-qi} to estimate the advantage based on a learned value model, and uses a clipped loss function to avoid excessive gradient updates.

{\bf Importance Weighting:} Because we use asynchronous RL (\autoref{sec:codebase}) with worker nodes performing rollouts and trainer nodes operating simultaneously, rollouts may be generated using a slightly out of date policy parameter $\theta_{\textup{old}}$, making them somewhat off-policy.
For limited degrees of staleness, a simple importance weighting correction suffices to make the gradients unbaised without introducing too much variance.
This is done by evaluating the policy probabilities $\pi_{\theta_{\textup{old}}}$ on the workers, and sending them to the trainers, which compute an importance weight $w_i = \pi_\theta(a_t | a_{<t}, x) / \pi_{\theta_{\textup{old}}}(a_t | a_{<t}, x)$ and multiplies the loss by this weight.

\subsection{Asynchronous On- \& Off-policy RL Framework}
\label{sec:codebase}

We perform asynchronous RL training in a distributed system consisting of trainer and worker nodes (each containing $8$ H100 GPUs), similar to \citet{gehring2024rlef, Noukhovitch2024-lf}. 

The workers continuously perform rollouts using a batched and throughput-optimized inference implementation and parallelized sandboxed code execution for reward evaluation.
The resulting trajectories (token sequences and rewards) are sent to the trainers.

The trainers collect a batch of trajectories (from workers and optionally from offline sources), and do a training update using one of the available loss functions.
In general this will require a forward pass on the latest policy and reference model, as well as the value model for PPO, followed by a backward pass on the latest policy and value model (if applicable).
After a certain number of steps (model\_update\_interval), we send new model weights from the trainers to the workers, in order to enable approximately on-policy rollouts.
Compared to a synchrounous approach -- where the same nodes alternate between collecting rollouts and doing training updates -- the asynchronous approach yields much higher throughput and continuously high GPU utilization.

The workers are configurable so that they can do on-policy rollouts (temperature $1$, using the latest policy), or use alternative decoding methods such as temperature and top-p sampling using the latest or reference policy.
Likewise, the trainers can be configured to use a mix of offline trajectories loaded from disk and online data produced by the workers.
Trainers can also be configured to use different loss functions (optionally depending on the source of data) including the ones discussed in \autoref{sec:objectives}.
When using on-policy losses, we set a low value for model\_update\_interval, and although there is no strict guarantee that training trajectories are strictly on-policy, we have found that low values such as 1 -- 4 lead to stable training.

\section{Related Work}

Our work is based on the Soft RL framework, also known as RL-as-inference, maximum-entropy RL and KL-divergence control \citep{Todorov2008-jv, ziebart2008maximum, Kappen2009-dm, Schulman2017-eq, Levine2018-sd, Lazaro-Gredilla2024-gl}.
It is the basis of several popular methods, such as Soft Q-Learning and Soft Actor-Critic \citep{Haarnoja2018-pa, Haarnoja2017-xe}.
Most recent works on RL for LLMs, including RLHF \citep{Ziegler2019-or, Ouyang2022-ml}, use the KL-regularized reward maximization soft RL objective.

The off-policy objectives we use can be understood either as $\Qs$-regression towards (future-corrected) empirical targets, or as path-consistency losses applied to the entire trajectory (since they relate the values $\widehat{Q}_0$ and $\Qs_T = r$ at the beginning and end of the trajectory to the policy log-probs along the trajectory).
There is however a fundamental difference to path-consistency learning as explored by \citet{Nachum2017-is}, because in their case the policy and value function are learned separately, and path-consistency must be enforced by a loss on every sub-trajectory $t_1, \ldots, t_k$.
In our case, path-consistency is satisfied on all sub-trajectories, and learning only happens on the entire trajectory (with the terminal $\Qs_T$-value replaced by the observed reward $r$).

Similarly, although we learn a soft Q function, SPO is different from Soft Q-Learning \citep{Haarnoja2017-xe}, because we do not use Bellman backups (since according to \autoref{thm:bellman}, all TD errors are zero).
Other works have explored Q-Learning with transformers \citep{Chebotar2023-tc}, but again these methods need to learn internal Bellman-consistency whereas for SPO it is guaranteed.

Expressions involving a sum of log-probability differences $\log \pi_\theta - \log \piref$ have appeared in a number of recent works, most notably DPO \citep{Rafailov2023-tr} and DRO \citep{Richemond2024-ds}.
However these methods differ in that DPO is a preference-optimization method (requiring preference pairs rather than rewards for training), and DRO is a purely offline training method.
Such sums have also been used in concurrent work to parameterize ``process reward models'' (a new name for advantage \citep{Schulman2016-qi}), which can be used to speed up learning in policy gradient methods \citep{Yuan2024-tg, Cui2025-fv}.
However, these works use the resulting PRM only to accelerate learning with policy gradients, rather than using it directly as a policy and value function as we do.
Indeed, the authors found that their PRMs perform poorly as policies, which is likely due to the fact that they omitted the critical $\widehat{\Qs}_0$ estimate, and use a different loss function.

Other recent works have noted that removing the value model from PPO can have benefits, both for reasons of memory and bias, in the context of LLM training.
Group Relative Policy Optimization \citep{Shao2024-vi} samples multiple independent trajectories and computes an advantage for each, which is then used in the clipped PPO loss as a replacement for the GAE estimate.
Similarly, VinePPO uses Monte-Carlo rollouts for advantage estimation at each reasoning step in a rollout \citep{Kazemnejad2024-gs}.
Although these works share the benefit of not requiring a separate value model, neither can leverage off-policy data.


\section{Experiments \& Results}
\label{sec:experiments}

Our experiments are centered around two main questions: first, is incorporating previous experiences during the training of SPO possible and beneficial, and secondly, can SPO learn a better and more diverse policy compared to PPO? For the first point, it is important that SPO can not only train stably on arbitrary offline data, but actually benefits from adding more offline data from diverse sources. 

{\bf Experiment setting:} We run all our experiments on the challenging CodeContests benchmark introduced by \citet{li2022competition}. In this task the LLM is presented with a code contest problem description, followed by a description of the input and outputs, and some example input and outputs. The LLM needs to generate a Python solution that solves the coding challenge correctly within the given memory and time constraints. A solution is deemed correct if it produces the correct output for all the test inputs, including inputs that are not shown in the prompt. 
We evaluate the models based on the average over the evaluation sets in CodeContests, and in TACO \citep{li2023taco}. In our experience adding the TACO tests has proven to be a more reliable way of tracking progress. We deduplicate the CodeContests training data against the TACO test data.
We evaluate the models with the pass@10 score, which measures whether out of 10 solutions, at least one solution passes all the tests. This metric favors models with diverse outputs compared to the more commonly seen pass@1. To compute the pass@10 we sample 20 responses with a temperature of 0.4 and top-p of 0.95, and use the estimator from \citep{li2022competition}.

We initialize all policies using Llama-3.1-8B-Instruct \citep{grattafiori2024llama3herdmodels}. The model is then trained for 100,000 steps with a batch size of 128. We use 64 GPUs to update the policy and 32-64 GPUs to generate new samples. We use the AdamW optimizer with a warm-up period of 200 steps, followed by a constant learning rate of $6 \times 10^{-8}$.

We test two settings for SPO that differ only in their training data: pure online vs a combination of online and offline data in equal proportions. Both SPO runs use a $\beta=1/\log(100000) \approx 1/11.5$, and use the Q-regression with cross-entropy loss. For each training problem we estimate $\Qs_0$ with Monte Carlo simulation by sampling 800 completions and evaluating their correctness.
In order to demonstrate SPO's ability to handle diverse data sources, we gather offline data from the \emph{reference model} (sampled from $\piref$ with random temperature and top-p 0.95), from an \emph{online SPO} run and \emph{PPO} run, samples generated by Llama-70B with a \emph{chain-of-though} (CoT) prompt, and 
\emph{human solutions} from the CodeContests dataset. We remove duplicates from the offline dataset and filter for correct solutions.
The online samples for SPO are generated from its current policy with a random temperature sampled from $\mathcal{U}(0.1, 0.8)$ and a top-p value of 0.95. This steers the training process towards sequences for which the model assigns higher probability. 

Our primary baseline is a well-tuned PPO implementation with a KL-regularized reward using the same $\beta$ as SPO. PPO is the most commonly used policy gradient method and it has also found popular application in training LLMs \citep{Ouyang2022-ml,gehring2024rlef}. It can train on slightly off-policy trajectories by applying importance sampling, but generally benefits from on-policy samples using the latest model. Hence we update the behavior policy after each batch of samples. This requires the trainer nodes to frequently send model updates to the worker nodes, which incurs a heavy slowdown in overall training speed (even though the transfer is done asynchronously). In contrast, in preliminary experiments with SPO we observed no significant benefits from on-policy samples or frequent model updates. Hence, we decrease the update frequency of the behavior policy to every 10 batches of samples. This has a significant effect on the training time, resulting in an {\bf 85\% speedup} in terms of wall clock time.

\begin{figure}
    \centering
    \includegraphics[width=0.85\columnwidth]
    {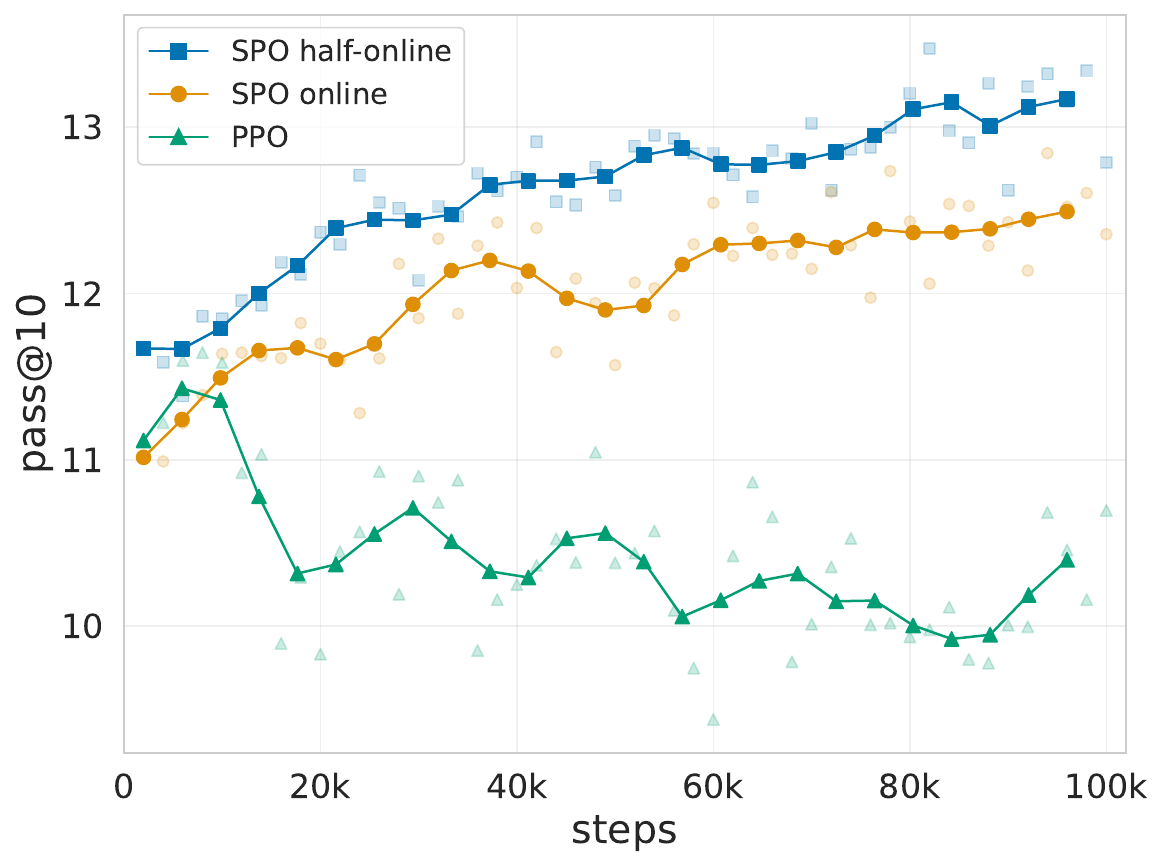}
    \caption{Code generation performance measured in pass@10 vs training step. SPO trained on mixed online and offline data outperforms both the pure online SPO and PPO. }
    \label{fig:SPO vs PPO}
\end{figure}

{\bf Results:} In \autoref{fig:SPO vs PPO}, the results show that the pass@10 performance of SPO steadily improves over training, while it declines somewhat for PPO.
Additionally, we observe that the half-online SPO run, which incorporates offline data, not only remains stable but also shows better results.
Importantly, both PPO and SPO-online require the same number of GPUs for training and collecting rollouts. In contrast, SPO half-online trains on 50\% offline samples, effectively reducing the sampling cost by half.

On pass@1, PPO performs better, yielding 8.4 vs 6.3 for SPO half-online and 6.0 for online SPO.

{\bf PPO discussion:} Our PPO baseline matches the pass@1 performance reported in \citep{gehring2024rlef} for the single-turn setup, but its improvement in pass@1 performance co-occurs with a modest decline in the pass@10 performance.
This happens because the PPO policy quickly learns to sample a single solution per problem, so the original difference between pass@1 and pass@10 performance diminishes over the duration of training. In contrast, SPO improves both while preserving the difference.

\subsection{Ablations and Extensions}

In our previous experiment we focused on a specific implementation of SPO: terminal Q-regression with cross-entropy loss. This allowed us to assess the practical feasibility of SPO by examining its training dynamics over long training runs. In this section we will ablate these choices in shorter training runs and explore additional directions to improve the performance of SPO. We consider terminal Q-regression with cross-entropy loss as the primary baseline, which we refer to simply as SPO. Unless otherwise specified, the ablation runs use the half-online setting.

\begin{figure*}[t!]
    \centering
    \subfigure[]{
        \includegraphics[width=0.31\textwidth]{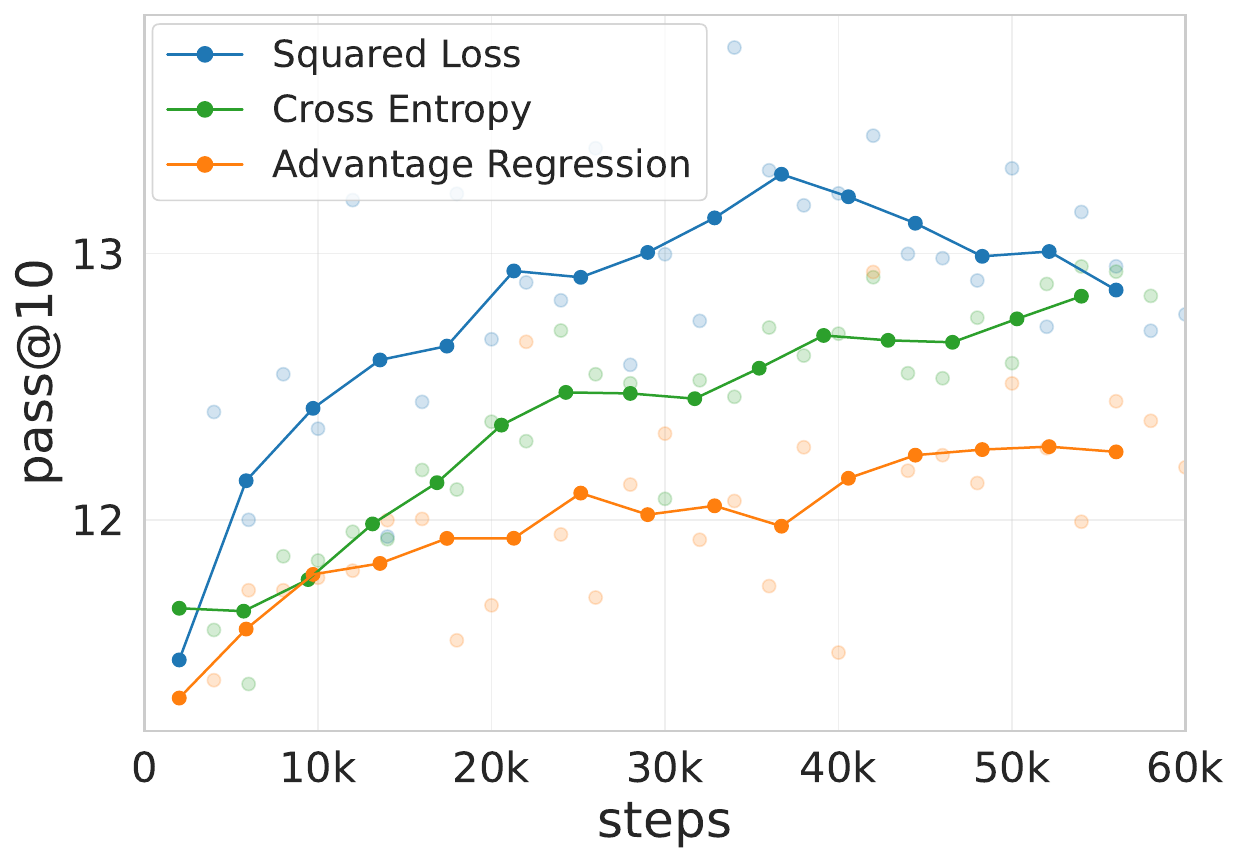}
        \label{fig:SPO_criterion}
    }
    \hfill
    \subfigure[]{
        \includegraphics[width=0.31\textwidth]{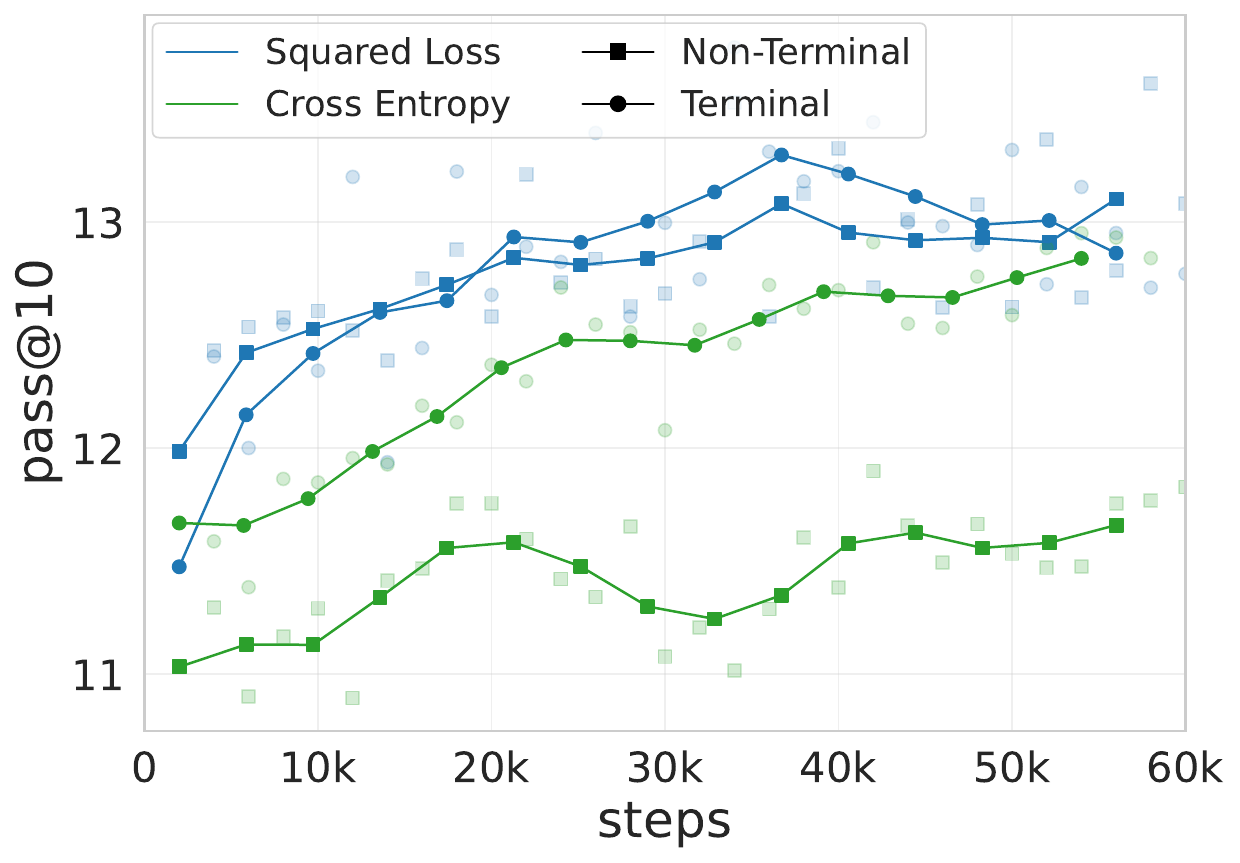}
        \label{fig:SPO_terminal_vs_non_terminal}
    }
    \hfill
    \subfigure[]{
        \includegraphics[width=0.31\textwidth]{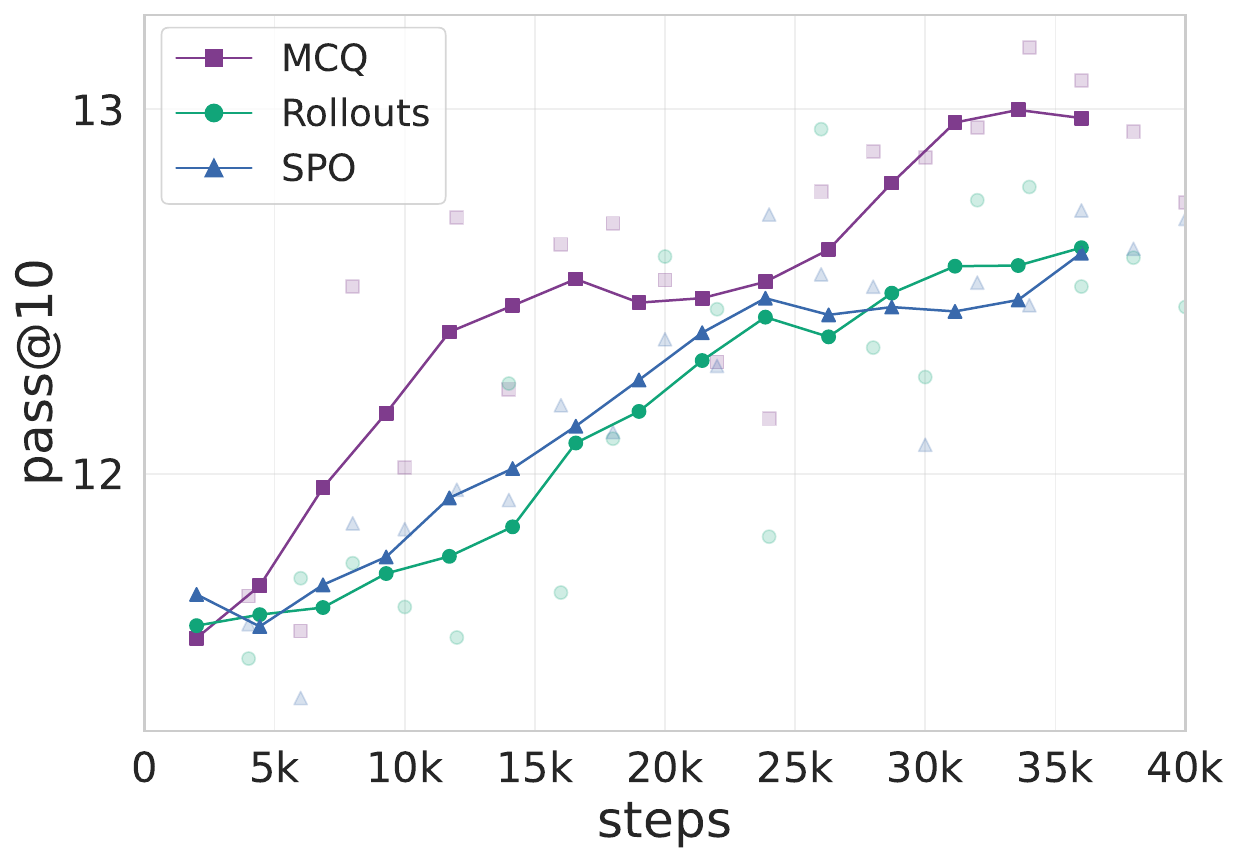}
        \label{fig:SPO_monte_carlo_targets}
    }
    \vspace{-1em}
    \caption{Ablation of different SPO settings and extensions: (a) Terminal Q-regression with different losses. (b) Terminal vs.\ non-terminal Q-regression with reverse-Q targets, combined with either squared or cross-entropy loss. (c) Non-terminal Q-regression with Monte Carlo Q-estimations at pivotal tokens, ablated against terminal Q-regression on the rollouts gathered during the pivotal token search.}
    \label{fig:combined_figure}
\end{figure*}

{\bf Loss function:} We begin by investigating the impact of different loss functions on performance.
We compare terminal Q-regression with squared loss and cross-entropy loss, and advantage regression with sigmoid warping and cross-entropy loss.
The results in \autoref{fig:SPO_criterion} show that the choice of loss function significantly affects the speed of improvement, with the squared loss outperforming the other two.

Interestingly, our findings differ from those of \citet{Yuan2024-tg}, who reported policy performance degradation when using a similar cross-entropy loss. This discrepancy can be attributed to the difference in target labels: their model was trained on response-level correctness labels, whereas our advantage regression loss correctly captures the mathematical relationship between the policy and the value function, leading to improved policy performance.

{\bf Terminal vs.\ Non-Terminal:} Next, we investigate whether a non-terminal loss using reverse-Q targets can enhance the performance. We conduct experiments with non-terminal Q-regression using both squared loss and cross-entropy loss. The results in \autoref{fig:SPO_terminal_vs_non_terminal} indicate that combining the non-terminal targets with cross-entropy loss significantly degrades the performance, while for the squared loss we observe no substantial effect.

{\bf Monte Carlo Targets:} Finally, we examine the effectiveness of using Monte Carlo estimation for non-terminal $\Qs_t$ in place of reverse-Q targets. We specifically target time steps $t$ that show the greatest increase in successive $\Qs_t$ values, identified through a Pivotal Token Search on multiple correct samples from our offline dataset (\autoref{sec:MC-targets}). This process produces 50,000 trajectories, annotated with the success probability estimated via 10 rollouts at each token visited by PTS. We  replace half of the offline data with these samples, resulting in a data distribution of 50\% online, 25\% standard offline, and 25\% samples with Monte Carlo targets.
We also introduce a baseline that replaces half of the offline data with rollouts obtained during the pivotal token search. This baseline checks whether integrating the additional information from rollouts with the terminal loss is sufficient, or if condensing these rollouts into $\Qs_t$ targets for non-terminal targets is essential. The results in \autoref{fig:SPO_monte_carlo_targets} show that the non-terminal Q-regression towards Monte Carlo targets significantly improves the performance. This happens despite only adding 50,000 trajectories, which corresponds to less than 1,600 steps per epoch. In contrast, applying the terminal loss solely on the rollouts performs the same as the baseline.

\section{Discussion}

The goal of Soft RL is to learn a policy that not only knows \emph{a} good response to every query, but ideally knows \emph{all} good responses to every query.
Although any KL-regularized reward maximization method shares this objective in theory, in practice policy gradient methods fail to preserve diversity in the policy.
They rapidly learn to increase the probability of responses that already have a reasonably high probability under the reference policy.
As a result they achieve good pass@1 performance early on, but it becomes increasingly hard to discover good responses that were not known already.
By contrast, even pure online SPO is able to steadily improve on both pass@1 and pass@10, with the gap remaining consistent or growing.
This effect is further enhanced by the inclusion of offline data containing solutions with relatively low probability under the reference model.

The cumulative Q parameterization provides a unified policy and value function that are guaranteed to be consistent.
This feature of our method enables many interesting possibilities for combined policy and value-based learning.
For instance, although still limited in scale, our experiments with Monte-Carlo $\Qs_t$ estimates show that combining terminal $\Qs$ loss with MC targets results in faster learning and improved final performance.
Other options, such as combining SPO with a policy gradient loss, remain to be explored as well.

One limitation of our current implementation is that we do not update the reference model.
Doing so would require online estimation of $\Qs_0$, but is likely to improve both pass@1 and pass@10 performance (preliminary experiments using fixed $\widehat{\Qs}_0$ support this notion).

In our experiments we have demonstrated the benefit of including offline data, but the full extent of the benefit of our approach will become evident only when we continuously grow the offline dataset as we run more and more experiments and ablations.
Although such a cumulative procedure becomes hard to reproduce, it drastically reduces the inference cost of RL training runs by amortizing it.

\section{Conclusion}

In this paper we have presented Soft Policy Optimization, a hybrid online off-policy RL method for language model improvement.
SPO is based on the \emph{cumulative Q-parameterization}, which has a number of appealing theoretical properties first revealed in this work.
Our large-scale experiments on code contests demonstrate that SPO is able to leverage diverse and highly off-policy data, and unlike PPO is able to preserve diversity in the policy throughout training, resulting in improved pass@10 scores.
Moreover, SPO is significantly more scalable and faster to train, due to a reduced need for model transfers, significantly reduced memory consumption, and reduced need for inference compute due to the use of offline data.
Although here we have experimented with a fixed set of offline data, the benefits of offline data are poised to grow as we accumulate diverse solutions to hard problems across the many training runs that are required for frontier model development.





\bibliography{refs}
\bibliographystyle{icml2025}

\end{document}